\documentclass[conference]{IEEEtran}
\IEEEoverridecommandlockouts

\usepackage{cite}
\usepackage{amsmath,amssymb,amsfonts}
\usepackage{amsthm}
\usepackage{algorithmic}
\usepackage{graphicx}
\usepackage{textcomp}
\usepackage{xcolor}
\usepackage{booktabs}
\usepackage{multirow}
\usepackage{bm}
\usepackage{bbm}
\usepackage{enumitem}
\usepackage[table]{xcolor}
\usepackage{colortbl}

\definecolor{lightblue}{RGB}{230,240,255}
\definecolor{lightred}{RGB}{255,200,200}

\newtheorem{theorem}{Theorem}

\newcommand{\RR}{\mathbb{R}}
\newcommand{\DD}{\mathbb{D}}

\newcommand{\sg}{\mathrm{sg}}
\newcommand{\Mob}{\oplus_c}
\newcommand{\hypdist}{d_{\mathbb{D}}}

\def\BibTeX{{\rm B\kern-.05em{\sc i\kern-.025em b}\kern-.08em
    T\kern-.1667em\lower.7ex\hbox{E}\kern-.125emX}}

\usepackage[hidelinks]{hyperref}

\begin{document}

\title{Learning Variable-Length Tokenization for\\ Generative Recommendation}

\author{\IEEEauthorblockN{Minhao Wang}
\IEEEauthorblockA{\textit{East China Normal University} \\
Shanghai, China \\
51275901104@stu.ecnu.edu.cn}
\and
\IEEEauthorblockN{Bowen Wu}
\IEEEauthorblockA{\textit{East China Normal University} \\
Shanghai, China \\
10245102410@stu.ecnu.edu.cn}
\and
\IEEEauthorblockN{Wei Zhang\textsuperscript{*}}
\IEEEauthorblockA{\textit{East China Normal University} \\
\textit{\& Shanghai Innovation Institute}\\
Shanghai, China \\
zhangwei.thu2011@gmail.com}
\thanks{\textsuperscript{*}Corresponding author.}
}

\maketitle

\begin{abstract}
Generative recommendation reformulates recommendation as next-token prediction over discrete semantic identifiers (IDs). A fundamental yet unexplored design choice is that existing methods employ fixed-length tokenization for all items, implicitly assuming uniform encoding capacity regardless of item characteristics. Through systematic experiments across four datasets, we discover the Popularity-Length Paradox: popular items achieve optimal performance with short IDs, while tail items require substantially longer codes to capture discriminative semantics. This reveals a critical mismatch where popular items benefit from abundant collaborative signals accumulated through frequent interactions, requiring minimal semantic detail, whereas tail items must rely primarily on fine-grained content features due to sparse interaction data. To address this, we propose VarLenRec, a framework for learning variable-length tokenization. We first develop Popularity-Weighted Information Budget Allocation (PIBA), an information-theoretic framework that proves optimal identifier length should scale as a negative power of popularity. However, directly implementing variable-length allocation faces two technical challenges: standard Euclidean residual quantization lacks geometric capacity to simultaneously support diverse code lengths without distortion, and discrete length decisions are non-trivial to predict reliably. We address these through Hyperbolic Residual Quantization, which leverages the exponential volume growth of the Poincar\'{e} ball to naturally stratify encoding capacity from compact origin regions suitable for popular items to exponentially expanding boundaries providing fine-grained capacity for tail items, and a Length Predictor, which casts length determination as a classification task supervised by PIBA-derived theoretical priors. Extensive experiments demonstrate that VarLenRec achieves significant improvements over state-of-the-art methods in recommendation accuracy and improves training and inference efficiency, revealing the fundamental importance of adaptive encoding capacity in generative recommendation.
\end{abstract}


\section{Introduction}
\label{sec:intro}

\begin{figure}[!t]
\centering
\includegraphics[width=0.95\linewidth]{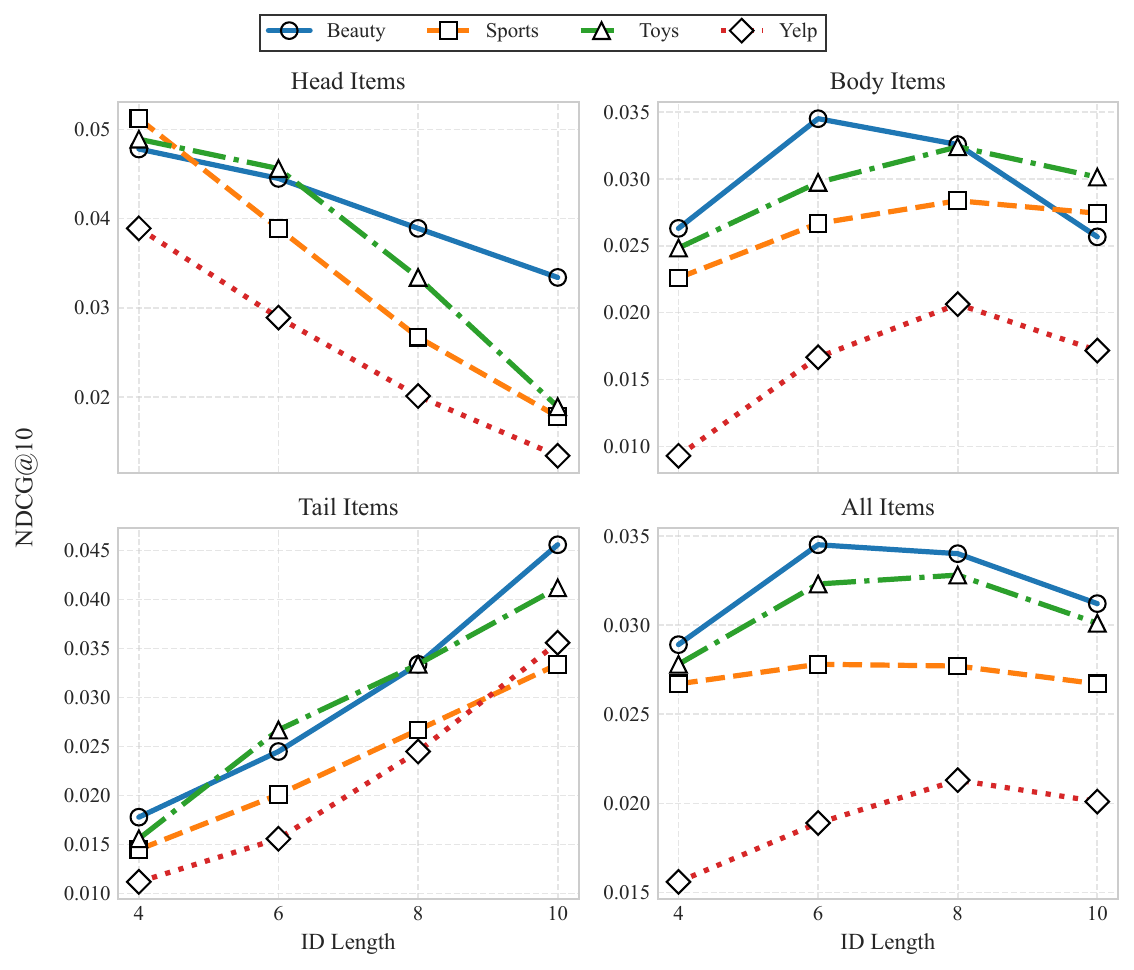}
\caption{The Popularity-Length Paradox. NDCG@10 on different target groups (Head: Top 20\%, Body: Mid 60\%, Tail: Bottom 20\%) across semantic ID lengths ($L=4,6,8,10$) on four datasets. Head items achieve peak performance at shorter lengths, while tail items benefit most from longer IDs.}
\label{fig:paradox}
\end{figure}

Recommender systems play a crucial role in helping users navigate information overload by providing personalized item suggestions~\cite{smith2017amazon,chen2019behavior,covington2016youtube}. With the advancement of large language models, generative recommendation (GR) has emerged as a promising paradigm that reformulates recommendation as next-token prediction~\cite{rajput2023tiger,zheng2024lcrec,li2024survey,wang2025towards,deng2025onerec}. In this framework, the model receives a sequence of item semantic IDs representing a user's interaction history and generates the IDs of the next item aligned with the user's interests.

A critical component in generative recommendation is the construction of high-quality semantic IDs. Existing methods employ various tokenization techniques to discretize continuous item embeddings into multi-layer token sequences, including parameter-free approaches like clustering or matrix factorization~\cite{hua2023cid,petrov2023gptrec,si2024seater,wang2024eager}, as well as deep learning methods based on residual quantization (RQ)~\cite{rajput2023tiger,wang2024letter,liu2025etegrec,zhang2025macrec,huang2025genrank,jia2025learn}. These methods have achieved promising results by incorporating collaborative signals, enabling end-to-end optimization, or leveraging cross-modal information.

However, a fundamental design choice remains unexplored: all existing methods assign fixed-length semantic IDs to every item, implicitly assuming that all items require identical encoding capacity. Recommender systems often exhibit long-tail distributions, where a small number of popular (head) items account for most interactions, while the majority of unpopular (tail) items have sparse data~\cite{park2008long,yin2012challenging}. This heterogeneity motivates us to investigate whether optimal ID length varies with item popularity. Through systematic empirical analysis, we partition test instances based on the popularity of their ground-truth target items (Head: top 20\%, Body: middle 60\%, Tail: bottom 20\%) and evaluate recommendation performance across different ID lengths ($L \in \{4, 6, 8, 10\}$). Our experiments reveal a striking phenomenon, termed the Popularity-Length Paradox (Figure~\ref{fig:paradox}): for head items, performance decreases with longer IDs; for tail items, performance increases with longer IDs. The globally optimal fixed length is merely a suboptimal compromise.

Our empirical analysis reveals an information sufficiency asymmetry across popularity levels~\cite{balabanovic1997fab,burke2002hybrid}. Head items consistently achieve peak performance with short IDs (L=4 yields best NDCG@10), with performance degrading as length increases. In contrast, tail items demonstrate the opposite trend, with NDCG@10 improving substantially from shorter to longer codes. This pattern suggests that frequent interactions provide sufficient collaborative signals for head items to be discriminated with compact representations, while tail items benefit from extended semantic capacity to capture discriminative content features.

To address this issue, we propose \textbf{VarLenRec}, a principled framework that learns variable-length semantic IDs tailored to each item's popularity. The core scientific problem is how to optimally allocate and integrate such variable-length IDs into generative recommendation systems.

This problem poses four interconnected technical challenges: (1) the lack of theoretical guidance for determining appropriate ID lengths, (2) standard Euclidean residual quantization provides insufficient geometric capacity to support diverse lengths without introducing significant distortion, (3) discrete length decisions need to be predicted reliably from item content, and (4) variable-length IDs introduce downstream generation issues, including ID collisions, hallucinations, and length determination during beam search~\cite{rajput2023tiger,deldjoo2024review}. VarLenRec addresses these challenges through four complementary components. First, we develop Popularity-Weighted Information Budget Allocation (PIBA), an information-theoretic method that theoretically proves optimal ID length scales as a negative power of popularity (Theorem~\ref{thm:optimal}), providing principled priors for length allocation. Second, building on these priors, we introduce Hyperbolic Residual Quantization in the Poincar\'{e} ball, which exploits the exponential volume growth of hyperbolic geometry~\cite{nickel2017poincare} to naturally stratify encoding capacity by placing popular (head) items in compact origin regions with short IDs and allocating expanding boundary regions for tail items that require longer, finer-grained IDs. We further provide a rigorous geometric analysis (Theorems~\ref{thm:capacity}--\ref{thm:radial}) that quantifies this capacity stratification, bounds the embedding distortion against Euclidean space, and links radial position to popularity. Third, to determine the length of each item from its content, we train a Length Predictor that casts length determination as a classification task supervised by PIBA-derived priors. Fourth, we incorporate downstream integration techniques, including collision resolution via disambiguation tokens, prefix-tree constrained decoding to prevent hallucinations, and a user-side length predictor that determines the target ID length before generation, ensuring variable-length IDs can be effectively utilized in autoregressive generation. These components are tightly coupled: PIBA provides the theoretical foundation and supervision signal for both the hyperbolic quantization design and the length predictor, while the downstream techniques ensure practical usability without compromising recommendation quality.

In summary, our main contributions are:
\begin{itemize}[leftmargin=*, noitemsep, topsep=3pt]
    \item We discover the Popularity-Length Paradox through systematic experiments across four datasets, revealing that uniform-length tokenization is fundamentally suboptimal due to popularity-based heterogeneity.
    \item We develop PIBA, an information-theoretic method that derives optimal length allocation in closed form as a negative power of popularity, and establish a geometric theory showing that hyperbolic space provides exponential, distortion-free capacity stratification matched to this allocation.
    \item We propose VarLenRec, a variable-length tokenization framework for adaptive item tokenization through hyperbolic residual quantization combined with a length predictor.
    \item Extensive experiments demonstrate substantial gains in recommendation accuracy (up to 8.3\% in NDCG@10 over the strongest baseline), collision reduction (from 12.7\% to 3.2\%), and training/inference efficiency thanks to shorter IDs for popular items.
\end{itemize}

\section{Related Work}
\label{sec:related}

In this section, we review the related work in two major aspects: sequential recommendation and generative recommendation.

\subsection{Sequential Recommendation}

Sequential recommendation aims to predict the next item a user may interact with based on the user's historical behavior sequences. Early studies~\cite{rendle2010fpmc,he2016fusing} primarily adhere to the Markov Chain assumption and focus on estimating the transition matrix. With the development of neural networks, various model architectures have been applied for sequential recommendation. Recurrent Neural Networks such as GRU4Rec~\cite{hidasi2016gru4rec} and NARM~\cite{li2017narm} learn sequential patterns and temporal dependencies through hidden state propagation. Convolutional approaches like Caser~\cite{tang2018caser} apply horizontal and vertical filters to extract sequential patterns at multiple scales.

Recently, transformer-based recommendation models have achieved great success in effective sequential user modeling. SASRec~\cite{kang2018sasrec} utilizes a Transformer decoder with unidirectional self-attention to capture user preference. BERT4Rec~\cite{sun2019bert4rec} proposes to encode the sequence by bidirectional attention and adopts the mask prediction task for training. S$^3$-Rec~\cite{zhou2020s3rec} explores using the intrinsic data correlation as supervised signals to pre-train the sequential model for better user and item representations. Furthermore, several works exploit the abundant textual and visual features of items to enrich representations. FDSA~\cite{zhang2019fdsa} separately models item-level and feature-level sequences using dual self-attention networks. MISSRec~\cite{wang2023missrec} employs multimodal pretraining to strengthen user interest modeling. These methods typically represent items as continuous embeddings and perform recommendation via similarity-based retrieval or ranking. In contrast, generative recommendation reformulates the task as autoregressive sequence generation over discrete semantic IDs, where each ID corresponds to a learnable continuous code (e.g., from a quantized codebook). Although the underlying representations remain continuous, the use of discrete tokens enables compatibility with off-the-shelf language models and facilitates end-to-end generation of item sequences.

\subsection{Generative Recommendation}

Generative recommendation has emerged as a next-generation paradigm for recommendation systems~\cite{li2024survey,deldjoo2024review}. In such a generative paradigm, the item sequence is tokenized into a token sequence and then fed into generative models to predict the tokens of the target item autoregressively. Generally, this paradigm involves two main processes: item tokenization and generative recommendation.

Existing approaches for item tokenization can be broadly categorized into parameter-free methods and deep learning methods based on multi-level vector quantization (VQ). For parameter-free methods, CID~\cite{hua2023cid} and GPTRec~\cite{petrov2023gptrec} apply matrix factorization to the co-occurrence matrix to derive item IDs. SEATER~\cite{si2024seater} and EAGER~\cite{wang2024eager} employ clustering of item embeddings to construct IDs hierarchically. Other works~\cite{geng2022p5,tan2024idgenrec} use the textual metadata attached to items as IDs. While these non-parametric methods are highly efficient, they often suffer from length bias and fail to capture deeper collaborative relationships among items.

Deep learning methods based on multi-level VQ develop more expressive item IDs via Deep Neural Networks. TIGER~\cite{rajput2023tiger} uses RQ-VAE to learn hierarchical semantic IDs from textual embeddings. LETTER~\cite{wang2024letter} aligns quantized embeddings with collaborative embeddings to leverage both collaborative and semantic information. ETEGRec~\cite{liu2025etegrec} achieves end-to-end joint optimization of item tokenization and generative recommendation through a recommendation-oriented alignment strategy. In multimodal settings, MMGRec~\cite{liu2024mmgrec} uses a Graph RQ-VAE to integrate multimodal features with collaborative signals. MACRec~\cite{zhang2025macrec} introduces cross-modal learning during quantization, enabling semantic IDs to capture the advantageous features of different modalities. MQL4GRec~\cite{zhai2025mql4grec} further advances the field by constructing a unified set of semantic IDs that jointly encode multimodal and cross-domain item information.

For the generative component, encoder-decoder architectures like T5~\cite{raffel2020t5} are widely adopted due to their excellent capabilities in sequence modeling and generation. P5~\cite{geng2022p5} unifies multiple recommendation tasks through natural language prompts. LC-Rec~\cite{zheng2024lcrec} utilizes the natural language understanding abilities of LLMs to support diverse task-specific fine-tuning. Recent industrial systems~\cite{wang2024kuaiformer,deng2025onerec,jia2025learn} demonstrate the effectiveness of unifying retrieval and ranking within generative frameworks, while hybrid architectures~\cite{pang2025gram,yang2025cobra} explore the integration of sparse and dense representations for improved scalability. Furthermore, LLM-based semantic embedding learning~\cite{hu2024enhancing} has shown promise in capturing fine-grained user preferences.

In summary, the dominant tokenization paradigm for generative recommendation treats semantic ID length as a fixed hyperparameter, assigning uniform-length IDs to all items regardless of their popularity or interaction patterns, thereby ignoring the heterogeneity in information requirements across items. A few methods do produce identifiers of unequal length, but only as an uncontrolled by-product: SEATER~\cite{si2024seater}, for instance, observes that na\"{i}ve generative retrieval ``yield[s] identifiers with inconsistent lengths,'' which it treats as a drawback that harms inference efficiency, and consequently enforces a \emph{balanced} $k$-ary tree to make all identifiers equal-length. In other words, prior work either fixes the length or actively removes length variation. In contrast, we treat length variation as a first-class design objective rather than an artifact to be eliminated: we deliberately allocate ID length to each item according to a theoretically grounded information budget tied to its popularity (PIBA), and realize this allocation through hyperbolic quantization. To our knowledge, this is the first work to \emph{learn popularity-adaptive, theoretically-allocated variable-length} semantic IDs for generative recommendation.

\begin{figure*}[t]
    \centering
    \includegraphics[width=\textwidth]{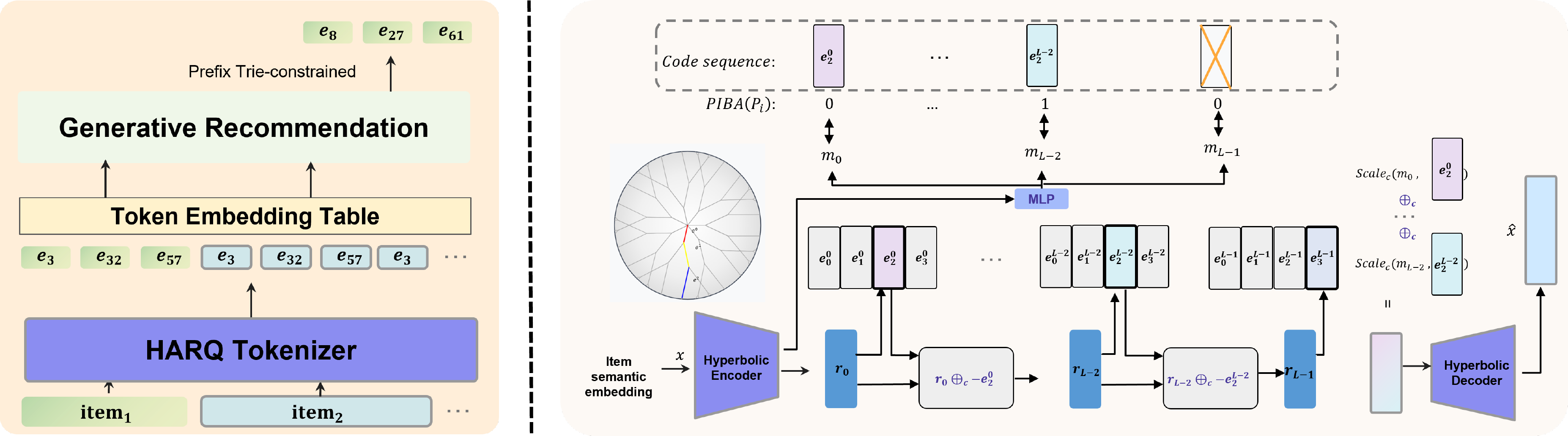}
    \caption{Overall architecture of VarLenRec.}
    \label{fig:model-overview}
\end{figure*}

\section{Methodology}
\label{sec:method}

We present our approach in four parts: the empirical discovery motivating our work (\S\ref{sec:discovery}), the theoretical foundation (\S\ref{sec:theory}), the Hyperbolic Adaptive Residual Quantization (HARQ) architecture (\S\ref{sec:model}), and the downstream generative recommendation (\S\ref{sec:downstream}). Figure~\ref{fig:model-overview} illustrates the complete architecture.

\subsection{The Popularity-Length Paradox}
\label{sec:discovery}

Let $\mathcal{I} = \{i_1, i_2, \ldots, i_N\}$ denote the item set. Each item $i \in \mathcal{I}$ has content features $\mathbf{x}_i \in \RR^F$ and popularity $p_i$, defined as the normalized interaction frequency in the training set. Standard RQ-VAE assigns each item a fixed-length semantic ID $\mathbf{z}_i = (z_i^{(1)}, z_i^{(2)}, \ldots, z_i^{(L)})$, where $z_i^{(l)} \in \{c_1^{(l)}, c_2^{(l)}, \ldots, c_M^{(l)}\}$ indexes the codebook at layer $l$ and $L$ is the uniform length for all items.

To investigate whether uniform length is appropriate, we conduct a stratified evaluation based on target item popularity. We train RQ-VAE models with lengths $L \in \{4, 6, 8, 10\}$ on the full training set, then partition test instances into three groups according to the popularity of their ground-truth target items: Head (top 20\%), Body (middle 60\%), and Tail (bottom 20\%). We compute NDCG@10 separately for each group.

Figure~\ref{fig:paradox} reveals a striking asymmetric pattern. For test instances targeting popular items (Head), performance monotonically decreases as length increases. Conversely, for test instances targeting unpopular items (Tail), performance monotonically increases with longer IDs. The Body group exhibits intermediate behavior, peaking at moderate length. The overall optimal length is merely a compromise that is suboptimal for every individual group.

This phenomenon, which we term the \textbf{Popularity-Length Paradox}, persists across multiple datasets (Figure~\ref{fig:paradox}). The underlying cause is an information sufficiency asymmetry: popular items accumulate abundant collaborative signals that make short IDs sufficient, while sparsely-interacted tail items need longer IDs to capture discriminative content features. This analysis motivates our approach: tokenization length should adapt to item popularity.

\subsection{Variable-Length Allocation}
\label{sec:theory}

We formalize the length allocation problem from an information-theoretic perspective and derive a principled mapping from item popularity to optimal ID length.

\subsubsection{Problem Formulation}

Consider an item catalog $\mathcal{I} = \{i_1, i_2, \ldots, i_N\}$ where each item $i$ has content features $\mathbf{x}_i \in \RR^F$ and popularity $p_i$ (normalized interaction frequency, $\sum_{i=1}^{N} p_i = 1$). Our goal is to tokenize each item into a semantic ID $\mathbf{z}_i = (z_i^{(1)}, \ldots, z_i^{(L_i)})$ through residual quantization, where $z_i^{(l)} \in \{c_1^{(1)}, \ldots, c_M^{(1)}\}$ indexes a codebook at layer $l$ and $L_i \in \{1, \ldots, K\}$ is the ID length.

Standard fixed-length quantization methods~\cite{oord2017vqvae,lee2022rqvae} assign uniform length $L$ to all items by maximizing mutual information:
\begin{equation}
\max_{\theta} I(X; Z) \quad \text{s.t.} \quad |Z| = L.
\label{eq:rqvae_obj}
\end{equation}
This approach ensures semantic fidelity but suffers from length rigidity, failing to account for heterogeneous information requirements across items with different popularity.

We reformulate the objective to enable variable-length allocation. Let $L_i$ denote the effective length for item $i$. We seek the optimal length assignment that maximizes semantic information while penalizing encoding cost:
\begin{equation}
\max_{L_i, \theta} I(X; Z_{1:L_i}) - \lambda \mathbb{E}[L_i]\,,
\label{eq:our_obj}
\end{equation}
where $Z_{1:L_i}$ denotes the ID truncated at length $L_i$, $\mathbb{E}[L_i]$ represents the expected length, and $\lambda > 0$ controls the trade-off between semantic fidelity and encoding cost.
Motivated by the Popularity-Length Paradox in Section~\ref{sec:discovery}, we propose an item-dependent allocation strategy guided by popularity. Our key insight is that popular items benefit from collaborative signals accumulated through frequent interactions, requiring less semantic information~\cite{balabanovic1997fab,burke2002hybrid}. Conversely, tail items lack reliable collaborative signals and must rely on fine-grained content semantics~\cite{park2008long}.

We adopt the hybrid recommendation framework where each item requires a minimum information budget $I_{\mathrm{req}} > 0$ for effective recommendation~\cite{balabanovic1997fab,burke2002hybrid}, representing the minimal bits necessary to uniquely identify an item, analogous to the minimum description length principle~\cite{gray1998quantization}. This budget is satisfied through two complementary sources. First, collaborative signals contribute:
\begin{equation}
I_{\mathrm{collab}}(p_i) = \alpha \log(1 + \theta p_i)\,,
\label{eq:collab_info}
\end{equation}
where $\log$ denotes the natural logarithm (consistent with the nats convention above), $\alpha > 0$ quantifies information gain rate per log-popularity and $\theta > 0$ scales baseline interaction frequency. The logarithmic form captures submodular information gain in collaborative filtering~\cite{krause2008near,koren2009matrix}, where initial interactions provide substantial information but additional interactions contribute progressively less due to redundancy~\cite{rashid2002getting}.

Second, semantic IDs contribute information through multi-level quantization. We measure all information quantities in nats for consistency. In residual quantization, each layer encodes the residual left by previous layers, whose energy decays geometrically across layers~\cite{lee2022rqvae,zeghidour2021soundstream}. Under a high-rate quantization model~\cite{gray1998quantization}, the information a layer contributes is proportional to the log of the residual magnitude it resolves; with geometrically shrinking residuals this yields a diminishing-returns profile in which the marginal information of the $l$-th layer behaves as $\gamma/l$, where $\gamma > 0$ is a base-capacity constant determined by the codebook (with $\gamma$ on the order of $\ln M$). The total semantic information for length $L_i$ is therefore approximated by the harmonic sum:
\begin{equation}
I_{\mathrm{semantic}}(L_i) = \gamma \sum_{l=1}^{L_i} \frac{1}{l} = \gamma H_{L_i} \approx \gamma \ln L_i \,,
\label{eq:semantic_info}
\end{equation}
where the asymptotic expansion of harmonic numbers $H_L \approx \ln L + \eta$ ($\eta \approx 0.5772$ the Euler--Mascheroni constant)~\cite{graham1994concrete} is used. We note that $\gamma/l$ is a tractable diminishing-returns model rather than an exact per-layer capacity; it captures the empirically observed saturation of deep RQ layers and is sufficient for deriving the scaling law below.

The information gap filled by semantic IDs is:
\begin{equation}
G_i = I_{\mathrm{req}} - \alpha \log(1 + \theta p_i)\,.
\label{eq:info_gap}
\end{equation}
Popular items have small gaps needing short codes; tail items have large gaps requiring long codes.

\subsubsection{Optimal Length and Practical Assignment}

Finding the minimum length satisfying $I_{\mathrm{semantic}}(L_i) \geq G_i$ yields $\gamma \ln L_i^* \approx I_{\mathrm{req}} - \alpha \log(1 + \theta p_i)$. For $\theta p_i \gg 1$, applying $\log(1 + \theta p_i) \approx \log \theta + \log p_i$ gives the following result.

\begin{theorem}[Optimal Length Allocation]
\label{thm:optimal}
Under the Information Budget framework, the optimal ID length satisfies
\begin{equation}
L_i^* = C \cdot p_i^{-\alpha/\gamma}\,,
\label{eq:optimal_length}
\end{equation}
where $C > 0$ is a constant. Thus, optimal length scales as a negative power of popularity.
\end{theorem}

\begin{proof}
Under the Information Budget framework, the semantic ID must fill the information gap $G_i$, requiring $I_{\mathrm{semantic}}(L_i) \geq G_i = I_{\mathrm{req}} - \alpha \log(1 + \theta p_i)$. Following Eq.~\eqref{eq:semantic_info}, we adopt the diminishing-returns model in which the $l$-th quantization layer contributes a marginal $\gamma/l$ nats of information, so that the total semantic information for an ID of length $L_i$ is $I_{\mathrm{semantic}}(L_i) = \gamma \sum_{l=1}^{L_i} \frac{1}{l} = \gamma H_{L_i}$, where $H_{L_i}$ denotes the $L_i$-th harmonic number. We emphasize that, consistent with Eq.~\eqref{eq:semantic_info}, the per-layer term $\gamma/l$ is a tractable approximation of the high-rate residual capacity rather than an exact bound; it suffices to derive the scaling law below, and all information quantities are measured in nats. Using the asymptotic expansion $H_L \approx \ln L + \eta$ (where $\eta \approx 0.5772$ is the Euler--Mascheroni constant), we have $I_{\mathrm{semantic}}(L_i) \approx \gamma \ln L_i$. The optimal length minimizes encoding cost while satisfying the constraint, achieved at equality:
\begin{equation}
\gamma \ln L_i^* = I_{\mathrm{req}} - \alpha \log(1 + \theta p_i)\,.
\end{equation}
Solving for $L_i^*$ gives $L_i^* = \exp\left( ( I_{\mathrm{req}} - \alpha \log(1 + \theta p_i) )/\gamma \right)$. For items with sufficient popularity such that $\theta p_i \gg 1$, we approximate $\log(1 + \theta p_i) \approx \log \theta + \log p_i$:
\begin{align}
L_i^* &= \exp\left( \frac{I_{\mathrm{req}} - \alpha \log \theta - \alpha \log p_i}{\gamma} \right) \notag \\
&= \exp\left( \frac{I_{\mathrm{req}} - \alpha \log \theta}{\gamma} \right) \cdot p_i^{-\alpha/\gamma}\,.
\end{align}
Defining $C = \exp\bigl( (I_{\mathrm{req}} - \alpha \log \theta) / \gamma \bigr) > 0$, we obtain $L_i^* = C \cdot p_i^{-\alpha/\gamma}$. Since $\alpha, \gamma > 0$, the exponent $-\alpha/\gamma < 0$, confirming that optimal length decreases with popularity.
\end{proof}

While this power law provides theoretical guidance, the continuous $L_i^*$ must be discretized to integers. We use rank-based quantile mapping: sort items by descending popularity, assign ranks $r_i \in \{0, \ldots, N-1\}$ where $r_i=0$ is most popular, normalize to coldness $q_i = r_i/(N-1)$, then apply a temperature transformation $\tilde{c}_i = q_i^{\beta}$ where $\beta > 0$ corresponds to $\alpha/\gamma$ in Equation~\eqref{eq:optimal_length}. The mapping to discrete lengths is given by:
\begin{equation}
L_i = \mathrm{clip}\left( \mathrm{round}\left( 1+(K-1)\left(\frac{r_i}{N-1}\right)^{\beta} \right), 1, K \right)\,.
\label{eq:final_assignment}
\end{equation}
This ensures robustness, maintains monotonicity $p_a > p_b \Rightarrow L_a \leq L_b$, and offers controllable granularity.

The discrete targets $\{L_i\}$ from Equation~\eqref{eq:final_assignment} serve as supervision signals for our Length Predictor (Section~\ref{sec:model}). We denote the PIBA-derived target length for item $i$ as $\hat{L}(p_i) \in \{1, \ldots, K\}$, which is treated as a categorical label over the $K$ possible lengths. We refer to this complete method as Popularity-weighted Information Budget Allocation (PIBA).

\subsection{HARQ Architecture}
\label{sec:model}

We present HARQ, which performs adaptive-length tokenization in hyperbolic space through residual quantization. It consists of a hyperbolic encoder, a hyperbolic residual quantizer, a length predictor, and a hyperbolic decoder.

\subsubsection{Why Hyperbolic Space}
\label{sec:why_hyp}

Item semantics follow hierarchical structures where broad categories branch into subcategories. This creates a geometric mismatch with Euclidean space: tree hierarchies contain exponentially many leaf nodes, while Euclidean balls provide only polynomial volume growth $V_{\mathrm{Eucl}}(r) \propto r^d$~\cite{krioukov2010hyperbolic}. In contrast, the Poincar\'{e} ball $\DD_c^d = \{\mathbf{x} \in \RR^d : \|\mathbf{x}\| < 1/\sqrt{c}\}$ with curvature $-c$ $(c>0)$ exhibits exponential volume growth $V_{\mathrm{Hyp}}(r) \propto \sinh^{d-1}(\sqrt{c}r) \approx e^{(d-1)\sqrt{c}r}$ for large $r$~\cite{ratcliffe2006foundations}, directly matching hierarchical branching and enabling distortion-free tree embeddings~\cite{sarkar2011low}.

This exponential expansion creates a natural capacity stratification for variable-length allocation. Near the origin where $r \approx 0$, volume remains compact and suitable for coarse-grained categories shared by popular items; as the radius increases toward the boundary where $r \to 1/\sqrt{c}$, volume expands exponentially, providing capacity for the fine-grained distinctions required by tail items. We now make this intuition precise. The following three theorems quantify the representational capacity at a given radius (Theorem~\ref{thm:capacity}), bound the tree-embedding distortion of hyperbolic versus Euclidean space (Theorem~\ref{thm:distortion}), and link radial position to popularity rank (Theorem~\ref{thm:radial}). Together they show that hyperbolic geometry is uniquely suited to realize the PIBA allocation $L_i^* \propto p_i^{-\alpha/\gamma}$. The theorems build on established results in hyperbolic geometry~\cite{ratcliffe2006foundations,sarkar2011low,bourgain1985lipschitz}.

\begin{theorem}[Hyperbolic Representational Capacity]
\label{thm:capacity}
In the Poincar\'{e} ball $\DD_c^d$ with curvature $-c$ $(c>0)$ and minimum separation $\delta > 0$, the maximum number of distinguishable regions within distance $r$ from the origin scales as
\begin{equation}
N(r) = \Theta\left( e^{(d-1)\sqrt{c}(r - \delta/2)} \right)\,,
\end{equation}
for $r \gg \delta$.
\end{theorem}

\begin{proof}
The volume of a hyperbolic ball of radius $r$ in $\DD_c^d$ is~\cite{ratcliffe2006foundations}
\begin{equation}
V_{\mathrm{hyp}}(r) = \omega_{d-1} \int_0^r \frac{\sinh^{d-1}(\sqrt{c}t)}{c^{(d-1)/2}} \, dt\,.
\end{equation}
For large $r$, this simplifies to
\begin{equation}
V_{\mathrm{hyp}}(r) \sim \frac{\omega_{d-1}}{2^{d-1}(d-1)c^{(d-1)/2}} e^{(d-1)\sqrt{c}r}\,,
\end{equation}
where $\omega_{d-1}$ is the $(d-1)$-sphere surface area. Each distinguishable region requires a separation ball of radius $\delta/2$ with volume $V_{\mathrm{hyp}}(\delta/2)$. Standard packing bounds~\cite{conway1999sphere} give
\begin{equation}
N(r) \leq \frac{V_{\mathrm{hyp}}(r + \delta/2)}{V_{\mathrm{hyp}}(\delta/2)} = \Theta\left( e^{(d-1)\sqrt{c}(r-\delta/2)} \right)\,.
\end{equation}
The lower bound matches the upper bound up to constant factors, which follows from horocyclic coordinate constructions~\cite{cannon1997hyperbolic}.
\end{proof}

\begin{theorem}[Hyperbolic vs.\ Euclidean Tree Distortion]
\label{thm:distortion}
Let $\Delta$ denote the maximum degree (branching factor) of a weighted tree $T$ with $n$ nodes and tree metric $d_T$. Then:
\begin{enumerate}[label=(\roman*), leftmargin=*, noitemsep]
\item A hyperbolic embedding $\phi_{\mathrm{hyp}}: T \to \DD_c^d$ with $d = O(\log \Delta)$ achieves constant distortion:
\begin{equation}
d_T(u,v) \leq \hypdist(\phi_{\mathrm{hyp}}(u), \phi_{\mathrm{hyp}}(v)) \leq (1+\epsilon) d_T(u,v)\,,
\end{equation}
for arbitrarily small $\epsilon > 0$.
\item Any Euclidean embedding $\psi_{\mathrm{eucl}}: T \to \mathbb{R}^d$ incurs distortion $\Omega(\log n)$.
\end{enumerate}
\end{theorem}

\begin{proof}
For the hyperbolic embedding, we follow Sarkar's construction~\cite{sarkar2011low}. Root the tree and assign depth-$k$ nodes coordinates
\begin{equation}
\phi_{\mathrm{hyp}}(u) = \tanh(\sqrt{c}k) \cdot \mathbf{v}_u\,,
\end{equation}
where $\mathbf{v}_u \in \mathbb{S}^{d-1}$ with disjoint angular sectors for children. For a parent--child pair at depths $k, k+1$, the hyperbolic distance is $\hypdist(\phi_{\mathrm{hyp}}(u), \phi_{\mathrm{hyp}}(v)) = 2$ after setting $c$ appropriately, matching the unit edge weight. The triangle inequality ensures the tree distance equals the sum of hyperbolic distances along paths. Embedding the up-to-$\Delta$ children of each node into disjoint angular sectors requires dimension $d = O(\log \Delta)$, which follows from sphere packing~\cite{conway1999sphere}.

For the Euclidean distortion, Bourgain's result~\cite{bourgain1985lipschitz} establishes the $\Omega(\log n)$ lower bound: a balanced binary tree of depth $h$ has $2^h$ leaves requiring separation $\Omega(h)$ in the tree metric but needing ball radius $\Omega(h \cdot 2^{h/d})$ in $\mathbb{R}^d$, yielding distortion $\Omega(2^{h/d}) = \Omega(n^{1/d})$. Optimizing over $d$ gives $\Omega(\log n)$.
\end{proof}

\begin{theorem}[Radial--Depth Correspondence]
\label{thm:radial}
In the hyperbolic tree embedding, depth-$k$ nodes satisfy $\hypdist(\mathbf{0}, \phi_{\mathrm{hyp}}(u)) = 2k$. For items modeled as a balanced tree with branching factor $b$, an item at popularity rank $r_i$ among $n$ items has expected depth
\begin{equation}
\mathbb{E}[k_i] = \log_b(n) - \log_b(r_i) + O(1)\,,
\end{equation}
implying that radial distance grows logarithmically with inverse popularity.
\end{theorem}

\begin{proof}
By construction, $\phi_{\mathrm{hyp}}(u) = \tanh(\sqrt{c}k) \cdot \mathbf{v}_u$ for unit $\mathbf{v}_u$. The distance from the origin is
\begin{align}
\hypdist(\mathbf{0}, \phi_{\mathrm{hyp}}(u)) &= \frac{2}{\sqrt{c}} \mathrm{arctanh}(\sqrt{c} \cdot \tanh(\sqrt{c}k)) \notag \\
&= 2k\,,
\end{align}
using $\mathrm{arctanh}(\tanh(x)) = x$. In a balanced tree with branching $b$, the depth of a leaf is $\log_b(n)$. The top-$r$ popular items form a complete subtree to depth $k_r \approx \log_b(r)$, giving expected depth $\mathbb{E}[k_i] \approx \log_b(n) - \log_b(r_i)$ for rank-$r_i$ items. Hence radial distance increases logarithmically as popularity decreases.
\end{proof}

Theorems~\ref{thm:capacity}--\ref{thm:radial} jointly justify our design: the radial coordinate naturally encodes hierarchy depth (Theorem~\ref{thm:radial}), deeper radii expose exponentially more discriminative capacity (Theorem~\ref{thm:capacity}), and this is achieved without the $\Omega(\log n)$ distortion that Euclidean space unavoidably incurs (Theorem~\ref{thm:distortion}). This places popular items near the compact origin with short IDs and tail items near the exponentially expanding boundary with long IDs, exactly realizing the PIBA prior $L_i^* \propto p_i^{-\alpha/\gamma}$.

\subsubsection{Hyperbolic Operations}

We use three operations on the manifold. M\"{o}bius addition:
\begin{equation}
\mathbf{x} \Mob \mathbf{y} = \frac{(1 + 2c\langle \mathbf{x}, \mathbf{y} \rangle + c\|\mathbf{y}\|^2)\mathbf{x} + (1 - c\|\mathbf{x}\|^2)\mathbf{y}}{1 + 2c\langle \mathbf{x}, \mathbf{y} \rangle + c^2\|\mathbf{x}\|^2\|\mathbf{y}\|^2}\,.
\end{equation}
Hyperbolic distance, exponential map, and its inverse logarithmic map:
\begin{align}
\hypdist(\mathbf{x}, \mathbf{y}) &= \tfrac{2}{\sqrt{c}} \mathrm{arctanh}\bigl(\sqrt{c}\|-\mathbf{x} \Mob \mathbf{y}\|\bigr)\,, \\
\exp_{\mathbf{0}}^c(\mathbf{v}) &= \tanh(\sqrt{c}\|\mathbf{v}\|) \tfrac{\mathbf{v}}{\sqrt{c}\|\mathbf{v}\|}\,, \\
\log_{\mathbf{0}}^c(\mathbf{x}) &= \tfrac{1}{\sqrt{c}} \mathrm{arctanh}(\sqrt{c}\|\mathbf{x}\|) \tfrac{\mathbf{x}}{\|\mathbf{x}\|}\,.
\end{align}
The logarithmic map converts hyperbolic representations back to Euclidean space for subsequent processing.

\subsubsection{Model Components}

The hyperbolic encoder maps item features to the Poincar\'{e} ball. Given features $\mathbf{x}_i \in \RR^F$, we obtain a Euclidean representation $\mathbf{h}_i = \mathrm{MLP}_{\mathrm{enc}}(\mathbf{x}_i) \in \RR^d$ and project onto the manifold: $\mathbf{z}_i^{(0)} = \exp_{\mathbf{0}}^c(\mathbf{h}_i) \in \DD_c^d$. The encoder learns to place popular items near the compact origin and distribute tail items toward the exponentially expanding boundary.

The hyperbolic residual quantizer performs $K$ layers of iterative quantization. Each layer $l$ maintains a codebook $\mathcal{E}^{(l)} = \{\mathbf{e}_1^{(l)}, \ldots, \mathbf{e}_M^{(l)}\}$ with code vectors in $\DD_c^d$. Initializing $\mathbf{r}_i^{(0)} = \mathbf{z}_i^{(0)}$, we select the nearest code at each layer using hyperbolic distance:
\begin{equation}
z_i^{(l)} = \arg\min_{j \in \{1,\ldots,M\}} \hypdist(\mathbf{r}_i^{(l-1)}, \mathbf{e}_j^{(l)})\,,
\end{equation}
and compute the hyperbolic residual via M\"{o}bius addition of the negated code:
\begin{equation}
\mathbf{r}_i^{(l)} = \bigl(-\mathbf{e}_{z_i^{(l)}}^{(l)}\bigr) \Mob \mathbf{r}_i^{(l-1)}\,.
\end{equation}

The length predictor is a separate lightweight classifier that determines how many quantization layers each item should retain directly from its content. Rather than making layer-wise gating decisions, we cast length determination as a single $K$-way classification problem. Given the Euclidean representation $\mathbf{h}_i = \mathrm{MLP}_{\mathrm{enc}}(\mathbf{x}_i)$ produced by the encoder, the predictor outputs a distribution over the $K$ possible lengths:
\begin{equation}
\mathbf{u}_i = \mathrm{softmax}\bigl(\mathrm{MLP}_{\mathrm{len}}(\mathbf{h}_i)\bigr) \in \Delta^{K-1}\,,
\end{equation}
where $\mathbf{u}_i = (u_i^{(1)}, \ldots, u_i^{(K)})$ and $u_i^{(k)}$ denotes the predicted probability that item $i$ should be assigned length $k$. At inference time, the discrete length is obtained as
\begin{equation}
L_i = \arg\max_{k \in \{1, \ldots, K\}} u_i^{(k)}\,.
\label{eq:len_argmax}
\end{equation}
This design decouples length determination from the quantization process: the predictor consumes only the encoder embedding, is trained with explicit supervision from the PIBA prior, and yields a single content-adaptive length decision per item.

The hyperbolic decoder reconstructs the item representation by sequentially accumulating the code vectors of the retained layers. Given the predicted length $L_i$, we initialize $\tilde{\mathbf{z}}_i^{(0)} = \mathbf{0}$ and update for the first $L_i$ layers:
\begin{equation}
\tilde{\mathbf{z}}_i^{(l)} = \tilde{\mathbf{z}}_i^{(l-1)} \Mob \mathbf{e}_{z_i^{(l)}}^{(l)}\,,\quad l = 1, \ldots, L_i\,.
\end{equation}
The final representation $\tilde{\mathbf{z}}_i = \tilde{\mathbf{z}}_i^{(L_i)}$ is mapped to Euclidean space via the logarithmic map: $\mathbf{y}_i = \log_{\mathbf{0}}^c(\tilde{\mathbf{z}}_i)$, and we reconstruct features: $\hat{\mathbf{x}}_i = \mathrm{MLP}_{\mathrm{dec}}(\mathbf{y}_i)$. During training, we supervise reconstruction using the PIBA target length $\hat{L}(p_i)$ to ensure that the first $\hat{L}(p_i)$ layers suffice to reconstruct the item.

\subsubsection{Training Objective}

The objective involves three loss terms. (\romannumeral1) The reconstruction loss measures fidelity, $\mathcal{L}_{\mathrm{recon}} = \|\mathbf{x}_i - \hat{\mathbf{x}}_i\|^2$.
(\romannumeral2) The quantization loss encourages encoder--codebook commitment~\cite{oord2017vqvae}:
\begin{equation}
\mathcal{L}_{\mathrm{quant}} = \sum_{l=1}^{K} \Big[ \hypdist\bigl(\sg[\mathbf{r}_i^{(l-1)}], \hat{\mathbf{e}}_i^{(l)}\bigr)^2 + \mu\, \hypdist\bigl(\mathbf{r}_i^{(l-1)}, \sg[\hat{\mathbf{e}}_i^{(l)}]\bigr)^2 \Big]\,,
\end{equation}
where $\sg[\cdot]$ is the stop-gradient operator, $\hat{\mathbf{e}}_i^{(l)} = \mathbf{e}_{z_i^{(l)}}^{(l)}$, and $\mu > 0$ is the commitment weight (set to $0.25$ in our experiments following~\cite{oord2017vqvae}).
(\romannumeral3) The length prediction loss trains the length predictor against the PIBA-derived target length via a standard cross-entropy classification objective:
\begin{equation}
\mathcal{L}_{\mathrm{len}} = -\sum_{k=1}^{K} \mathbbm{1}\bigl[k = \hat{L}(p_i)\bigr] \log u_i^{(k)}\,,
\label{eq:len_ce}
\end{equation}
where $\hat{L}(p_i)$ is the PIBA target length from Equation~\eqref{eq:final_assignment} treated as a categorical label, and $u_i^{(k)}$ is the predicted probability of length $k$. This loss directly teaches the predictor to map item content to the theoretically optimal length.

The full objective is as follows:
\begin{equation}
\mathcal{L} = \mathcal{L}_{\mathrm{recon}} + \mathcal{L}_{\mathrm{quant}} + \lambda_{\mathrm{len}} \mathcal{L}_{\mathrm{len}}\,.
\end{equation}

\subsubsection{Riemannian Optimization}
\label{sec:riemannian}

Codebook vectors reside on the Poincar\'{e} manifold, so standard gradient descent may produce updates that leave the valid region. We employ Riemannian Adam~\cite{becigneul2019riemannian}, which respects the manifold geometry. The Poincar\'{e} ball has Riemannian metric tensor $g_{\mathbf{x}} = \lambda_{\mathbf{x}}^2 \mathbf{I}_d$, where $\lambda_{\mathbf{x}} = 2/(1 - c\|\mathbf{x}\|^2)$ is the conformal factor. The Riemannian gradient is obtained by rescaling the Euclidean gradient,
\begin{equation}
\nabla_{\mathrm{R}} \mathcal{L} = \frac{1}{\lambda_{\mathbf{w}}^2} \nabla_{\mathrm{E}} \mathcal{L} = \frac{(1 - c\|\mathbf{w}\|^2)^2}{4} \nabla_{\mathrm{E}} \mathcal{L}\,,
\end{equation}
and parameter updates use the exponential map $\mathbf{w}_{t+1} = \exp_{\mathbf{w}_t}^c(-\eta \nabla_{\mathrm{R}} \mathcal{L})$ to remain on the manifold, where
\begin{equation}
\exp_{\mathbf{x}}^c(\mathbf{v}) = \mathbf{x} \Mob \Bigl( \tanh\Bigl( \tfrac{\sqrt{c} \lambda_{\mathbf{x}} \|\mathbf{v}\|}{2} \Bigr) \tfrac{\mathbf{v}}{\sqrt{c}\|\mathbf{v}\|} \Bigr)\,.
\end{equation}
For numerical stability we enforce a safety margin: if the updated parameter violates $\|\mathbf{w}_{t+1}\| < 1/\sqrt{c} - \epsilon$ (with $\epsilon = 10^{-5}$), we project it back via radial rescaling $\mathbf{w}_{t+1} \leftarrow \frac{1/\sqrt{c} - \epsilon}{\|\mathbf{w}_{t+1}\|} \mathbf{w}_{t+1}$, which preserves the angular direction while keeping all codebook vectors within the valid ball interior and preventing division by zero in the conformal factor $\lambda_{\mathbf{x}}$.

\subsection{Variable-Length ID Integration}
\label{sec:downstream}

After training HARQ, we integrate the variable-length semantic IDs into the downstream Transformer-based generative recommendation model. During training, given a user's interaction history as a sequence of item IDs, we train the model to predict the next item in an autoregressive, token-by-token manner. The model naturally handles variable-length targets by continuous generation until the EOS token is emitted.

However, during inference, variable-length IDs introduce several technical challenges. A primary issue is ID collisions, where distinct items are assigned identical token sequences. In fixed-length settings, collisions can be resolved by appending a disambiguation token from the main codebook without affecting other items. This fails in the variable-length regime because appending a token from $\{1, \dots, M\}$ to a short ID may create a new sequence that coincides with a naturally longer one (e.g., $(3,7,5)$), and the appended token would corrupt training of deeper codebook layers by being treated as a semantic signal. To avoid this, we reserve a disjoint auxiliary codebook $\{M+1, \dots, 2M\}$ exclusively for post-hoc disambiguation. When collisions occur, such as two items sharing $(3,7)$, we append unique tokens from this range, yielding $(3,7,M+1)$ and $(3,7,M+2)$. This ensures global uniqueness without interfering with semantic quantization.
To prevent hallucination during generation, we construct a prefix tree (Trie) $\mathcal{T}$ over all collision-resolved IDs. At each decoding step, we perform constrained beam search that only explores tokens corresponding to valid children in $\mathcal{T}$, ensuring every generated sequence maps to an actual item.

Another critical challenge unique to variable-length IDs is determining how many tokens to generate for the target item. Standard beam search must decide when to terminate, and a length-agnostic decoder may generate sequences whose length does not match the target item assigned by HARQ. To resolve this, we introduce a user-side length predictor that determines the target ID length before generation and constrains beam search to decode exactly that many tokens.

Concretely, given the user interaction history $\mathbf{z}_{S_u}$, the T5 encoder produces a sequence of contextualized representations; we take the first user embedding $\mathbf{g}_u$ from the encoder output as a summary of the user's interest. A lightweight classification head maps this embedding to a distribution over the $K$ possible target lengths:
\begin{equation}
\mathbf{v}_u = \mathrm{softmax}\bigl(\mathrm{MLP}_{\mathrm{tlen}}(\mathbf{g}_u)\bigr) \in \Delta^{K-1}\,,
\end{equation}
and the predicted target length is obtained as $\hat{L}_u = \arg\max_{k} v_u^{(k)}$. This head is trained jointly with the generative model using a cross-entropy loss against the HARQ-assigned length of the ground-truth target item $i^*$:
\begin{equation}
\mathcal{L}_{\mathrm{tlen}} = -\sum_{k=1}^{K} \mathbbm{1}\bigl[k = L_{i^*}\bigr] \log v_u^{(k)}\,,
\label{eq:tlen_ce}
\end{equation}
so that the generative training objective becomes $\mathcal{L}_{\mathrm{gen}} + \lambda_{\mathrm{tlen}} \mathcal{L}_{\mathrm{tlen}}$, where $\mathcal{L}_{\mathrm{gen}}$ is the standard next-token generation loss.

During inference, we first predict the target length $\hat{L}_u$ from the user embedding, and then perform Trie-constrained beam search that decodes exactly $\hat{L}_u$ tokens, terminating once the predicted length is reached. By fixing the decoding length in advance, all candidate IDs within a beam share the same length, so their log-probability scores are directly comparable without any length penalty or rescoring. This length-conditioned generation ensures that longer IDs assigned to tail items and shorter IDs assigned to popular items are each retrieved at their correct length, while keeping the scoring of competing candidates fair by construction.

\noindent\textbf{Deployment discussion.}
Our variable-length tokenization framework is designed with practical enterprise deployment in mind. For cold-start items lacking interaction history, the system conservatively assigns the maximum ID length to ensure sufficient semantic capacity for discrimination based on content features alone. In real-world systems where item popularity naturally evolves over time, enterprise deployments can adopt a periodic update strategy to re-allocate ID lengths based on accumulated interaction statistics. This update process only requires re-running the lightweight HARQ tokenization module while keeping the generative model frozen, making it computationally efficient for large-scale systems. The framework's reliance on semantic features ensures that items with shifting popularity or newly introduced items can be appropriately encoded as long as their content descriptions are available. This design strikes a practical balance between adaptivity and stability, making VarLenRec a viable solution for production recommender systems.

\begin{table*}[!t]
\centering
\caption{Overall performance comparison. The best and second-best results are marked in bold and underlined, respectively.}
\label{tab:results}
\setlength{\tabcolsep}{3pt}
\renewcommand{\arraystretch}{1.45}
\resizebox{\textwidth}{!}{%
\begin{tabular}{l|cccc|cccc|cccc|cccc}
\toprule
\multirow{2}{*}{\textbf{Method}} & \multicolumn{4}{c|}{\textbf{Beauty}} & \multicolumn{4}{c|}{\textbf{Sports and Outdoors}} & \multicolumn{4}{c|}{\textbf{Toys and Games}} & \multicolumn{4}{c}{\textbf{Yelp}} \\
\cmidrule(lr){2-5} \cmidrule(lr){6-9} \cmidrule(lr){10-13} \cmidrule(lr){14-17}
 & R@5 & R@10 & N@5 & N@10 & R@5 & R@10 & N@5 & N@10 & R@5 & R@10 & N@5 & N@10 & R@5 & R@10 & N@5 & N@10 \\
\midrule
\rowcolor{lightblue}
\multicolumn{17}{c}{\emph{\textbf{Item ID-based}}} \\
\midrule
HGN & 0.0238 & 0.0419 & 0.0140 & 0.0200 & 0.0200 & 0.0359 & 0.0130 & 0.0181 & 0.0224 & 0.0313 & 0.0143 & 0.0170 & 0.0132 & 0.0221 & 0.0084 & 0.0111 \\
GRU4Rec & 0.0238 & 0.0333 & 0.0168 & 0.0198 & 0.0080 & 0.0126 & 0.0054 & 0.0069 & 0.0213 & 0.0297 & 0.0143 & 0.0180 & 0.0184 & 0.0310 & 0.0115 & 0.0156 \\
BERT4Rec & 0.0268 & 0.0447 & 0.0197 & 0.0281 & 0.0148 & 0.0246 & 0.0088 & 0.0125 & 0.0246 & 0.0297 & 0.0174 & 0.0249 & 0.0189 & 0.0315 & 0.0111 & 0.0158 \\
SASRec & 0.0307 & 0.0508 & 0.0199 & 0.0264 & 0.0185 & 0.0298 & 0.0117 & 0.0153 & 0.0356 & 0.0540 & 0.0246 & 0.0304 & 0.0153 & 0.0246 & 0.0098 & 0.0126 \\
FMLP & 0.0391 & 0.0600 & 0.0209 & 0.0276 & 0.0205 & 0.0317 & 0.0109 & 0.0145 & 0.0448 & 0.0631 & 0.0239 & 0.0298 & 0.0219 & 0.0326 & 0.0125 & 0.0163 \\
S$^3$Rec & 0.0434 & 0.0723 & 0.0309 & 0.0441 & 0.0267 & 0.0445 & 0.0175 & 0.0250 & 0.0471 & 0.0785 & 0.0283 & 0.0405 & 0.0245 & 0.0408 & 0.0144 & 0.0206 \\
HSTU & 0.0358 & 0.0574 & 0.0229 & 0.0299 & 0.0232 & 0.0360 & 0.0152 & 0.0194 & 0.0370 & 0.0559 & 0.0260 & 0.0320 & 0.0231 & 0.0403 & 0.0145 & 0.0200 \\
\midrule
\rowcolor{lightblue}
\multicolumn{17}{c}{\emph{\textbf{Semantic ID-based}}} \\
\midrule
TIGER & 0.0418 & 0.0617 & 0.0276 & 0.0340 & 0.0293 & 0.0409 & 0.0190 & 0.0277 & 0.0402 & 0.0583 & 0.0273 & 0.0331 & 0.0253 & 0.0407 & 0.0164 & 0.0213 \\
LC-Rec & 0.0467 & 0.0685 & 0.0329 & 0.0394 & 0.0318 & 0.0471 & 0.0209 & 0.0301 & 0.0435 & 0.0638 & 0.0294 & 0.0385 & 0.0241 & 0.0381 & 0.0158 & 0.0201 \\
ETEGRec & 0.0452 & 0.0691 & 0.0312 & 0.0387 & 0.0305 & 0.0432 & 0.0198 & 0.0285 & 0.0423 & 0.0601 & 0.0295 & 0.0367 & 0.0261 & 0.0415 & 0.0172 & 0.0220 \\
LETTER-TIGER & 0.0431 & 0.0672 & 0.0286 & 0.0364 & 0.0311 & 0.0456 & 0.0205 & 0.0292 & 0.0418 & 0.0625 & 0.0288 & 0.0373 & 0.0277 & 0.0426 & 0.0184 & 0.0231 \\
LETTER-LCRec & 0.0501 & \underline{0.0703} & 0.0355 & \underline{0.0418} & 0.0338 & 0.0489 & 0.0221 & 0.0315 & 0.0447 & 0.0652 & \underline{0.0301} & 0.0398 & 0.0255 & 0.0393 & 0.0168 & 0.0211 \\
\rowcolor{black!10}
VarLenRec-TIGER & \textbf{0.0533} & \textbf{0.0709} & \textbf{0.0377} & \textbf{0.0421} & \textbf{0.0349} & \textbf{0.0498} & \textbf{0.0228} & \textbf{0.0322} & \underline{0.0479} & \underline{0.0801} & 0.0293 & \underline{0.0412} & \textbf{0.0298} & \textbf{0.0458} & \textbf{0.0192} & \textbf{0.0246} \\
\rowcolor{black!10}
VarLenRec-LCRec & \underline{0.0519} & 0.0696 & \underline{0.0365} & 0.0412 & \underline{0.0342} & \underline{0.0491} & \underline{0.0223} & \underline{0.0317} & \textbf{0.0503} & \textbf{0.0846} & \textbf{0.0309} & \textbf{0.0431} & \underline{0.0285} & \underline{0.0434} & \underline{0.0188} & \underline{0.0233} \\
\bottomrule
\end{tabular}
}
\end{table*}

\section{Experiments}

In this section, we conduct comprehensive experiments to answer the following research questions. \textbf{RQ1:} How does integrating VarLenRec into state-of-the-art generative recommendation frameworks improve performance compared to both traditional sequential recommenders and existing fixed-length generative baselines? \textbf{RQ2:} How do different components of VarLenRec contribute to its performance? \textbf{RQ3:} What is the impact of VarLenRec on semantic ID quality, including collisions, learned length distribution, and per-popularity performance? \textbf{RQ4:} How sensitive is VarLenRec to its key hyperparameters, and how efficient is it in terms of training, inference, and decoding steps?

\subsection{Experimental Setup}

\subsubsection{Datasets}
We evaluate our method on four widely-used benchmark datasets: Amazon Beauty, Sports and Outdoors, Toys and Games~\cite{ni2019justifying}, and Yelp. Following previous work~\cite{rajput2023tiger,wang2024letter}, we filter out users and items with fewer than 5 interactions.

\subsubsection{Baselines}
We evaluate VarLenRec by integrating it with generative recommendation backbones and comparing against three categories of baselines. The first category includes traditional sequential recommendation methods: HGN~\cite{ma2019hgn}, GRU4Rec~\cite{hidasi2016gru4rec}, BERT4Rec~\cite{sun2019bert4rec}, SASRec~\cite{kang2018sasrec}, FMLP~\cite{zhou2022fmlp}, S$^3$Rec~\cite{zhou2020s3rec}, and HSTU~\cite{zhai2024hstu}. The second category contains generative recommendation methods with fixed-length semantic IDs: TIGER~\cite{rajput2023tiger} and ETEGRec~\cite{liu2025etegrec}. The third category includes enhanced generative recommendation methods: LETTER~\cite{wang2024letter} which improves codebook learning, applied to both TIGER and LCRec~\cite{zheng2024lcrec} backbones.

\subsubsection{Evaluation Metrics}
We adopt two widely-used metrics, Recall@K (R@K) and NDCG@K (N@K), with K $\in \{5, 10\}$. All metrics are computed using the full ranking protocol without negative sampling.

\subsubsection{Implementation Details}
For dataset splitting, we follow the leave-one-out strategy: the most recent item for each user is used as the test set, the second most recent item as the validation set, and all remaining items as the training set.
For fair comparisons, we use the same text encoder (Sentence-T5) to extract item embeddings for all generative methods. The codebook size is set to 256 with a maximum semantic ID length of 10. The hyperbolic curvature parameter is set to $c = 1.0$. The hyperparameters $\beta$, $\lambda_{\mathrm{len}}$, and $\lambda_{\mathrm{tlen}}$ are analyzed in Section~\ref{sec:hyperparam} and their settings are based on the results on the validation datasets. The training process consists of two stages. In the HARQ tokenization stage, we use the Riemannian Adam optimizer with learning rate 1e-4 and batch size 256 to train HARQ together with its length predictor. In the generative recommendation stage, we use T5 as the backbone architecture and train the Transformer-based generative model together with the user-side length predictor using the standard Adam optimizer with learning rate 1e-4 and batch size 256, and use beam search with beam size 30 during inference. All experiments are conducted on NVIDIA A800 GPUs.

\subsection{Overall Performance (RQ1)}

Table~\ref{tab:results} presents the performance comparison across all datasets. VarLenRec-TIGER consistently outperforms the fixed-length TIGER baseline, and VarLenRec-LCRec surpasses LCRec across all datasets. These improvements demonstrate that our variable-length tokenization mechanism is orthogonal to existing generative frameworks and can effectively enhance different backbones. The gains are particularly pronounced on datasets with diverse item categories (Beauty and Toys), where items benefit most from adaptive encoding capacity allocation.

When integrated with generative backbones, VarLenRec outperforms all existing semantic ID-based methods. VarLenRec-TIGER achieves the best results on three datasets (Beauty, Sports, Yelp), while VarLenRec-LCRec performs best on Toys. Against ETEGRec and LETTER-TIGER, both VarLenRec variants show considerable improvements. These results validate that addressing the Popularity-Length Paradox through adaptive tokenization provides complementary benefits to existing optimization strategies, regardless of whether they focus on joint training or representation alignment.

VarLenRec substantially outperforms all traditional sequential recommendation methods. Compared to the best baseline S$^3$Rec and recent advanced methods like HSTU, our best VarLenRec variant achieves notable improvements across all datasets. These results demonstrate the advantages of the generative paradigm with discrete semantic IDs, which enable explicit modeling of item semantics and hierarchical relationships. The variable-length mechanism further amplifies these advantages by tailoring encoding capacity to item-specific information requirements.

\subsection{Ablation Study }
To investigate the contribution of each component in VarLenRec, we conduct comprehensive ablation studies on VarLenRec-TIGER. We examine three key design choices: (1) the PIBA-based length prediction loss $\mathcal{L}_{\mathrm{len}}$, (2) the hyperbolic embedding space, and (3) the inference integration strategies (collision resolution, Trie-constrained decoding, and user-side length prediction). The results on all four datasets are shown in Table~\ref{tab:ablation}.
 
\begin{table*}[t]
\centering
\caption{Ablation study results across all four datasets. We ablate three key components: the length prediction objective $\mathcal{L}_{\mathrm{len}}$, hyperbolic space, and inference integration strategies.}
\label{tab:ablation}
\setlength{\tabcolsep}{12pt}
\renewcommand{\arraystretch}{0.95}
\resizebox{\textwidth}{!}{%
\begin{tabular}{l c c @{\hspace{3.2em}} c c @{\hspace{3.2em}} c c @{\hspace{3.2em}} c c}
\toprule
\multirow{2}{*}{\textbf{Variants}} & \multicolumn{2}{c}{\textbf{Beauty}} & \multicolumn{2}{c}{\textbf{Sports}} & \multicolumn{2}{c}{\textbf{Toys}} & \multicolumn{2}{c}{\textbf{Yelp}} \\
\cmidrule(lr){2-3} \cmidrule(lr){4-5} \cmidrule(lr){6-7} \cmidrule(lr){8-9}
 & R@10 & N@10 & R@10 & N@10 & R@10 & N@10 & R@10 & N@10 \\
\midrule
\rowcolor{black!10}
VarLenRec-TIGER & \textbf{0.0709} & \textbf{0.0421} & \textbf{0.0498} & \textbf{0.0322} & \textbf{0.0801} & \textbf{0.0412} & \textbf{0.0458} & \textbf{0.0246} \\
\midrule
\multicolumn{9}{l}{\textit{Fixed Length Baselines (w/o $\mathcal{L}_{\mathrm{len}}$ \& Length Predictor) }} \\
\quad Fixed Length = 4 & 0.0618 & 0.0356 & 0.0421 & 0.0271 & 0.0579 & 0.0327 & 0.0372 & 0.0201 \\
\quad Fixed Length = 6 & 0.0649 & 0.0378 & 0.0447 & 0.0289 & 0.0608 & 0.0341 & 0.0391 & 0.0212 \\
\quad Fixed Length = 8 & 0.0627 & 0.0364 & 0.0433 & 0.0279 & 0.0633 & 0.0353 & 0.0418 & 0.0227 \\
\quad Fixed Length = 10 & 0.0593 & 0.0332 & 0.0409 & 0.0264 & 0.0566 & 0.0319 & 0.0402 & 0.0218 \\
\multicolumn{9}{l}{\textit{Embedding Space Ablation}} \\
\quad w/ Euclidean Space & 0.0631 & 0.0372 & 0.0476 & 0.0307 & 0.0768 & 0.0396 & 0.0421 & 0.0224 \\
\midrule
\multicolumn{9}{l}{\textit{Inference Integration Ablation}} \\
\quad w/o Inference Strategies & 0.0667 & 0.0391 & 0.0451 & 0.0292 & 0.0739 & 0.0381 & 0.0439 & 0.0235 \\
\midrule
\multicolumn{9}{l}{\textit{Length Prediction Ablation}} \\
\quad w/ Direct PIBA Assignment & 0.0684 & 0.0403 & 0.0468 & 0.0301 & 0.0701 & 0.0359 & 0.0430 & 0.0229 \\
\bottomrule
\end{tabular}
}
\end{table*}
 
\noindent\textbf{Fixed Length Comparison.} Across all four datasets, no single fixed length is optimal: Beauty and Sports peak at a medium length (6), whereas Toys and Yelp peak at a longer length (8), consistent with the latter catalogs needing finer-grained codes. In every case the best fixed length still trails our adaptive approach by a clear margin, validating our hypothesis that items require different semantic granularities.
 
\noindent\textbf{Length Prediction Ablation.} Replacing the learned length predictor with the fixed PIBA-assigned lengths lowers NDCG@10 by 4.3\%, 6.5\%, 12.9\%, and 6.9\% on Beauty, Sports, Toys, and Yelp, respectively. The effect is highly dataset-dependent: on Toys it is by far the most damaging ablation, indicating that content-aware length prediction is the dominant ingredient for catalogs whose items demand fine-grained codes, while on Beauty its removal is comparatively mild.
 
\noindent\textbf{Embedding Space Ablation.} This variant is identical to full VarLenRec---retaining the Length Predictor and the PIBA-based length objective $\mathcal{L}_{\mathrm{len}}$---and only replaces the hyperbolic operations with standard Euclidean RQ-VAE quantization, thereby isolating the effect of geometry alone. It reduces Recall@10 by 11.0\% on Beauty and 8.1\% on Yelp, but only 4.4\% on Sports and 4.1\% on Toys. The hyperbolic geometry is thus the single most important component on Beauty and Yelp, whereas it plays a more supporting role on Sports and Toys, showing that geometry and length prediction contribute complementary benefits whose relative weight shifts with catalog structure.
 
\noindent\textbf{Inference Integration Ablation.} Removing all three inference-time strategies (collision resolution, Trie-constrained decoding, and user-side length prediction) simultaneously reduces Recall@10 by 5.9\%, 9.4\%, 7.7\%, and 4.1\% on Beauty, Sports, Toys, and Yelp. On Sports this is the most damaging of the three component ablations, indicating that correct length-conditioned decoding matters most for that catalog, while on the other datasets it provides a consistent, moderate contribution to the accurate and efficient retrieval of variable-length IDs during inference.

\subsection{Semantic ID Quality Analysis (RQ3)}
\label{sec:semantic_analysis}

Beyond overall performance, we analyze VarLenRec's impact on semantic ID quality through collision rates, length distribution, and performance across popularity groups.

\subsubsection{Collision Rate Analysis}

A critical quality metric for semantic IDs is the collision rate, i.e., the percentage of item pairs assigned identical ID sequences. Figure~\ref{fig:collision} compares collision rates between the standard RQ-VAE baseline (TIGER with fixed length L=10), our method with Euclidean space, and our full VarLenRec method across four datasets.

\begin{figure}[t]
\centering
\includegraphics[width=\columnwidth]{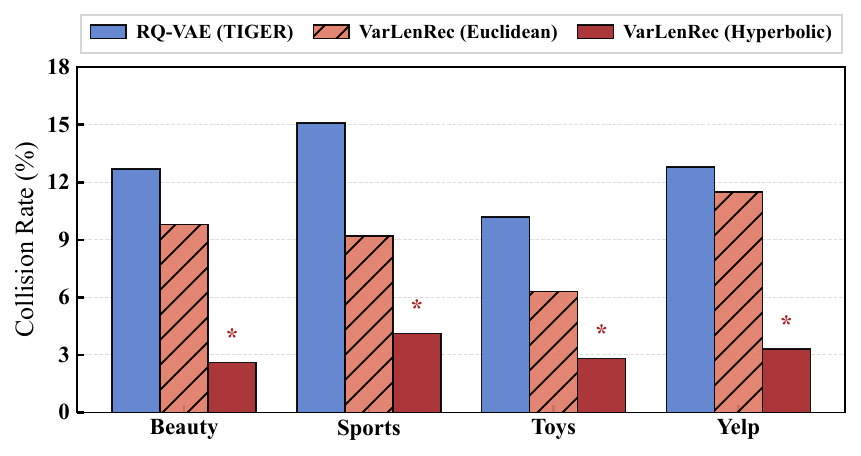}
\caption{ID collision rates (\%) across four datasets.
Our full VarLenRec with HARQ consistently achieves the lowest collision rates.}
\label{fig:collision}
\end{figure}

The full VarLenRec model with HARQ substantially reduces ID collisions compared to the fixed-length RQ-VAE baseline and the Euclidean variant. Across the four datasets, collision rates drop to a very low level (typically below 7\%), demonstrating that the combination of variable-length allocation and hyperbolic geometry effectively increases the expressiveness and distinctiveness of the learned semantic IDs. On average, collision rates drop from 12.7\% to 3.2\%.

\subsubsection{Semantic ID Length Distribution}

Figure~\ref{fig:length_dist} shows the distribution of learned semantic ID lengths across the four datasets, with a maximum length of 10.
The learned lengths vary across datasets according to how much semantic detail their items require, with most items receiving moderate-length IDs and the average length shifting from one catalog to another.

\begin{figure}[t]
\centering
\includegraphics[width=\columnwidth]{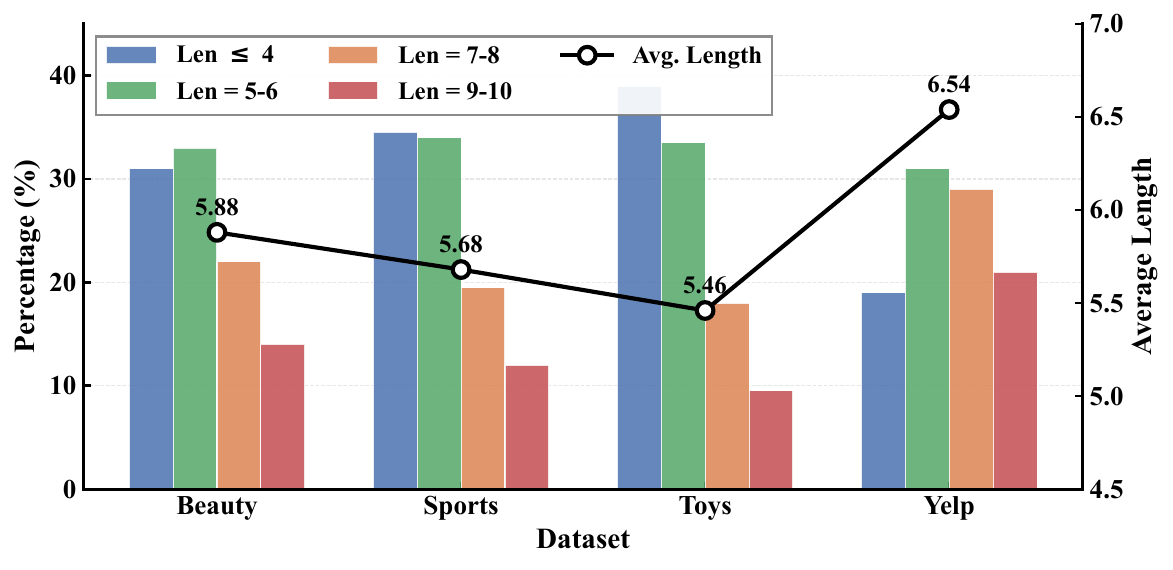}
\caption{Distribution of learned semantic ID lengths by VarLenRec across the four datasets (maximum length = 10). The bars show the percentage of items in each length range, while the line plot indicates the average length for each dataset.}
\label{fig:length_dist}
\end{figure}

Among the datasets, Yelp exhibits the highest average length (6.54), with the largest share of long IDs (50.0\% with length $\geq$ 7), which aligns with the greater semantic diversity of business entities compared to product categories.
In contrast, Toys shows the shortest average length (5.46) and the largest proportion of short IDs (39.0\% with length $\leq$ 4), reflecting its relatively standardized product categories.
Sports (5.68) and Beauty (5.88) fall between these two extremes, with Beauty skewing somewhat longer (36.0\% of items at length $\geq$ 7) than Sports (31.5\%), consistent with Beauty's more diverse product descriptions. The consistent gap between datasets confirms that VarLenRec adapts its length budget to each catalog rather than converging to a single global length.
This pattern suggests that VarLenRec adaptively assigns longer IDs to items requiring more fine-grained semantic distinctions, while simpler or more homogeneous items receive shorter codes.

\subsubsection{Performance Scaling by Popularity}

To understand how VarLenRec benefits different item groups, we analyze performance across popularity tiers with varying maximum semantic ID lengths. Figure~\ref{fig:pop_scaling} presents NDCG@10 results on the Beauty dataset.

\begin{figure}[t]
\centering
\includegraphics[width=\columnwidth]{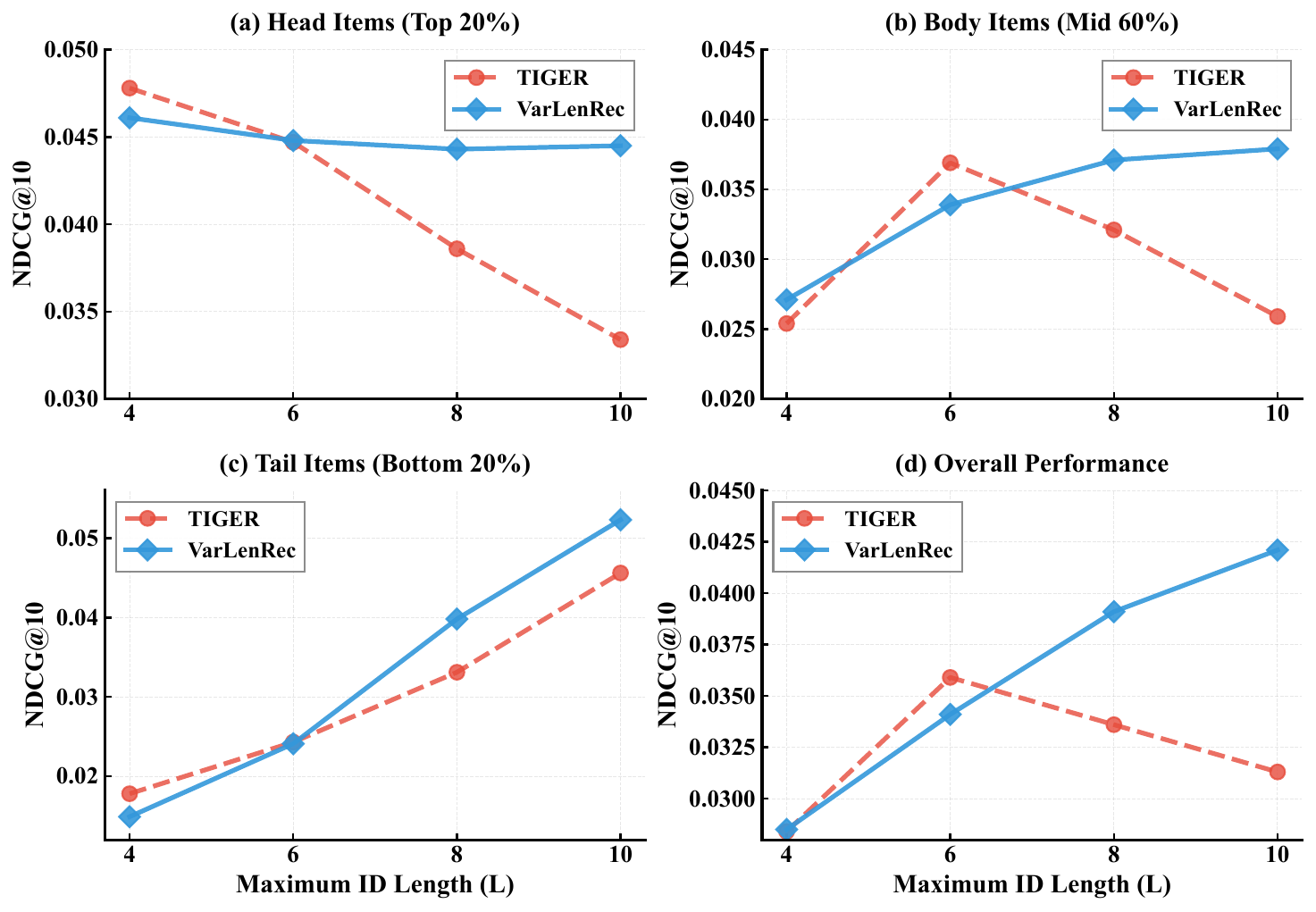}
\caption{Performance scaling analysis by item popularity across maximum ID lengths on Beauty dataset. VarLenRec mitigates the Popularity-Length Paradox observed in TIGER: (a) head items retain a strong short-length advantage rather than collapsing, (b) body items show consistent improvement, (c) tail items achieve substantial gains as the maximum length grows, and (d) overall performance rises with longer maximum budgets, demonstrating the effectiveness of adaptive length allocation.}
\label{fig:pop_scaling}
\end{figure}

TIGER exhibits the Popularity-Length Paradox: head items peak at short lengths (L=4: 0.0478) but degrade 30.1\% at L=10, while tail items show opposite behavior, improving 156.2\% from L=4 to L=10. No single length optimizes all groups simultaneously.

VarLenRec mitigates this paradox through adaptive length allocation. Because each item receives its own length, head items stay consistently strong across maximum lengths (0.0461 at L=4) rather than collapsing as in TIGER, while tail items gain +251.0\% and body items +39.9\% as the maximum length grows to L=10. As a result, overall performance climbs +47.7\% from L=4 to L=10, validating Theorem~\ref{thm:optimal}'s principle that encoding capacity should adapt to item-specific information requirements.

\subsection{Hyperparameter Sensitivity (RQ4)}
\label{sec:hyperparam}

We investigate VarLenRec's sensitivity to three key hyperparameters: the parameter $\beta$ in PIBA (Equation~\eqref{eq:final_assignment}), the length prediction weight $\lambda_{\mathrm{len}}$ in HARQ, and the user-side length prediction weight $\lambda_{\mathrm{tlen}}$ in the generative stage. Figure~\ref{fig:hyperparam} shows NDCG@10 across varying values on Beauty and Toys datasets.

\begin{figure}[t]
\centering
\includegraphics[width=\columnwidth]{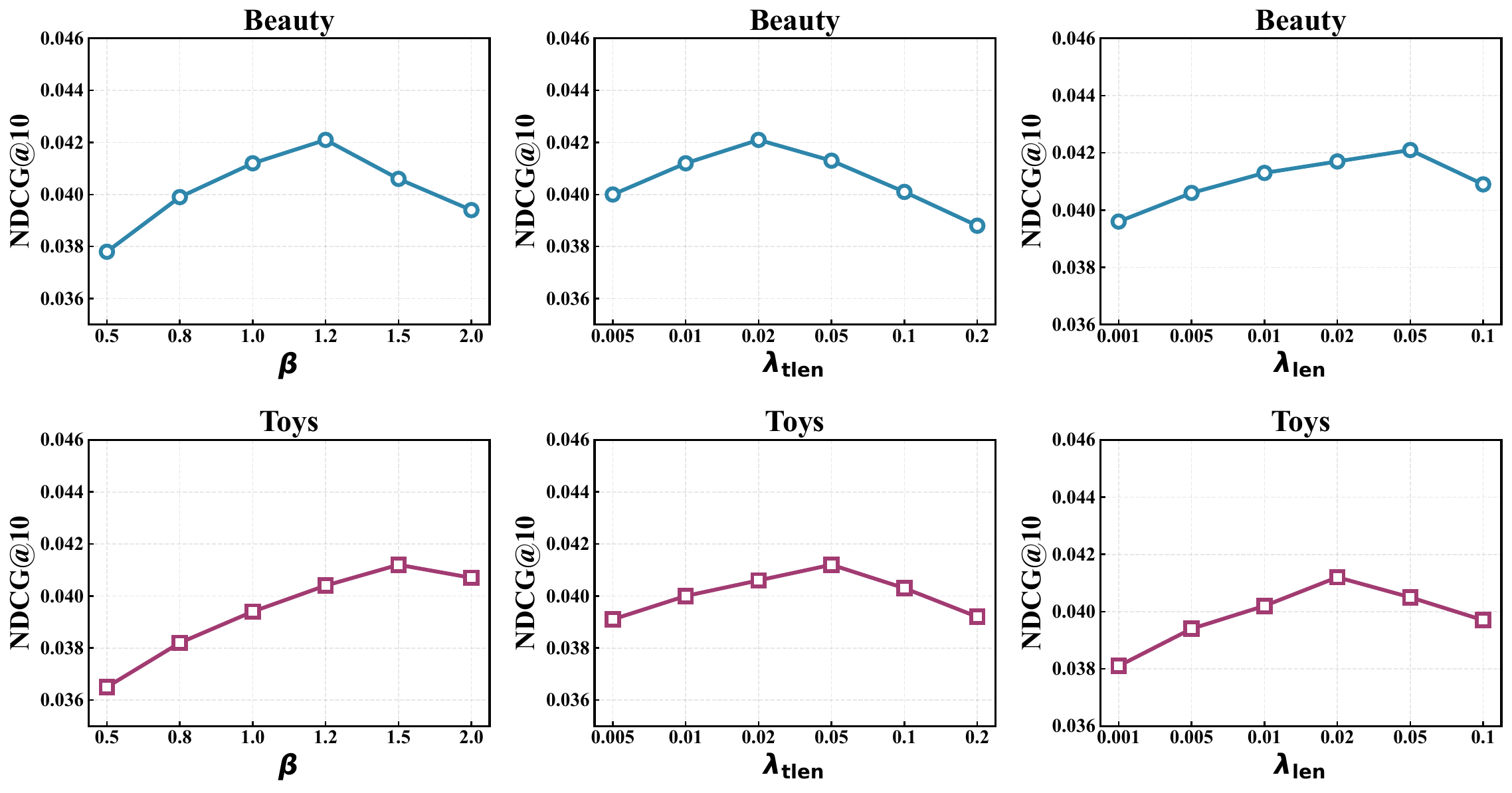}
\caption{Hyperparameter sensitivity analysis.}
\label{fig:hyperparam}
\end{figure}

\noindent\textbf{Parameter $\beta$.} Beauty achieves optimal performance at $\beta = 1.2$ with a clear peak pattern, while Toys peaks at $\beta = 1.5$ and maintains relatively stable performance at higher values.

\noindent\textbf{Length Prediction Weight $\lambda_{\mathrm{len}}$.} Beauty reaches its optimal performance at $\lambda_{\mathrm{len}} = 0.05$, while Toys achieves optimal performance earlier at $\lambda_{\mathrm{len}} = 0.02$. Too small values provide insufficient supervision for the length predictor, leaving length determination poorly aligned with the PIBA prior, while too large values over-emphasize length classification at the expense of reconstruction and quantization quality.

\noindent\textbf{User-side Length Prediction Weight $\lambda_{\mathrm{tlen}}$.} Beauty optimizes at $\lambda_{\mathrm{tlen}} = 0.02$, while Toys benefits from stronger regularization at $\lambda_{\mathrm{tlen}} = 0.05$. Very small values yield an unreliable target-length predictor that frequently decodes IDs at the wrong length, while excessive values divert capacity from the generation objective. The divergent optimal configurations across datasets demonstrate that hyperparameters should be tuned to specific catalog characteristics for optimal performance.

\subsection{Efficiency Analysis (RQ4)}

\noindent\textbf{Complexity.}
All RQ-based tokenizers (TIGER, LETTER, ETEGRec) share the dominant $O(NKMd)$ quantization cost; our length predictor adds only a lower-order $O(NKd_g)$ term ($d_g\!\ll\!Md$) for a single forward pass of a lightweight classifier. The length decision is a one-shot $\arg\max$ over $K$ classes with no search or backtracking. Since tokenization is a one-time offline step, this overhead is amortized to zero at serving time. The actual speedup comes from generation: with the user-side length predictor fixing the decoding length to $\hat{L}_u$, beam search costs $O(b\bar{L}|\mathcal{V}|)$ with average length $\bar{L}\!<\!K$, versus $O(bK|\mathcal{V}|)$ for fixed-length baselines, since popular items receive short IDs and terminate early.

\begin{table}[!t]
\centering
\caption{Training and testing efficiency comparison (seconds per epoch) across the four datasets. VarLenRec consistently achieves faster training and inference.}
\label{tab:efficiency}
\resizebox{\columnwidth}{!}{%
\begin{tabular}{lcccccccc}
\toprule
Method          & \multicolumn{2}{c}{Beauty} & \multicolumn{2}{c}{Sports} & \multicolumn{2}{c}{Toys} & \multicolumn{2}{c}{Yelp} \\
\cmidrule(lr){2-3} \cmidrule(lr){4-5} \cmidrule(lr){6-7} \cmidrule(lr){8-9}
                & Train & Test  & Train & Test  & Train & Test  & Train & Test  \\
\midrule
TIGER           & 37.1  & 87.2  & 45.3  & 106.8 & 35.8  & 82.4  & 52.6 & 124.3 \\
LETTER-TIGER    & 40.4  & 92.4  & 49.2  & 113.5 & 38.9  & 87.6  & 57.1 & 131.8 \\
\midrule
\rowcolor{black!10}
VarLenRec-TIGER & \textbf{30.4} & \textbf{63.7} & \textbf{39.9} & \textbf{81.2} & \textbf{32.6} & \textbf{70.0} & \textbf{37.9} & \textbf{83.3} \\
\bottomrule
\end{tabular}
}
\end{table}

Table~\ref{tab:efficiency} shows that VarLenRec-TIGER consistently achieves the fastest training and inference speeds across all datasets. On average, VarLenRec reduces training time by 16.7\% and testing time by 24.7\% compared to TIGER. The savings vary across datasets in line with their characteristics: Yelp, whose items carry the longest semantic IDs, benefits most (up to 33.0\% faster inference), whereas Toys, with the shortest IDs, sees more modest gains. This efficiency gain stems from the variable-length mechanism: popular items are assigned shorter IDs, substantially reducing the average number of tokens processed during both training and generation. The speedup is particularly pronounced during inference, where the predicted target length lets short IDs terminate early and eliminates unnecessary decoding steps. These computational savings make VarLenRec attractive for large-scale deployment scenarios where both training cost and inference latency are critical.

\subsection{Decoding-Step Efficiency (RQ4)}
\label{sec:decoding}

Beyond wall-clock time, variable-length tokenization directly reduces the number of autoregressive decoding steps at inference. Since a fixed-length tokenizer must emit $K$ tokens for every item regardless of popularity, while VarLenRec emits only $L_i$, the decoding cost scales with the average ID length rather than the maximum. Because autoregressive generation produces one token per step, the average ID length is exactly the average number of decoding steps per item. Table~\ref{tab:decoding} reports, for each dataset, the average decoding length broken down by popularity tier, together with the reduction in decoding steps relative to a fixed-length $L=10$ baseline.

\begin{table}[!t]
\centering
\caption{Decoding-step efficiency of VarLenRec. We report the average ID length (equivalently, the average number of decoding steps) per popularity tier (Head/Body/Tail), the overall average, and the reduction in decoding steps versus a fixed-length ($L=10$) baseline. Shorter IDs for popular items sharply cut the per-item decoding work.}
\label{tab:decoding}
\resizebox{\columnwidth}{!}{%
\begin{tabular}{l|ccc|c|c}
\toprule
\multirow{2}{*}{\textbf{Dataset}} & \multicolumn{3}{c|}{\textbf{Avg. Decoding Steps per Tier}} & \multirow{2}{*}{\shortstack{\textbf{Avg.}\\\textbf{Steps}}} & \multirow{2}{*}{\shortstack{\textbf{Step}\\\textbf{Red.}}} \\
\cmidrule(lr){2-4}
 & Head & Body & Tail & & \\
\midrule
Beauty & 3.1 & 6.03 & 8.2 & 5.88 & -41.2\% \\
Sports & 3.6 & 5.90 & 7.1 & 5.68 & -43.2\% \\
Toys   & 3.3 & 5.73 & 6.8 & 5.46 & -45.4\% \\
Yelp   & 4.5 & 6.43 & 8.9 & 6.54 & -34.6\% \\
\bottomrule
\end{tabular}
}
\end{table}

VarLenRec lowers the average number of decoding steps from the fixed $K=10$ to between 5.46 (Toys) and 6.54 (Yelp), a $35$--$45\%$ reduction in per-item decoding work whose magnitude varies across datasets according to how much semantic detail each catalog requires. The tier breakdown reveals where this saving originates: head items, which dominate inference traffic, require only $3.1$--$4.5$ steps on average, so the bulk of queries terminate early, whereas the longer codes (up to $8.9$ steps for Yelp tail items) are reserved for the comparatively rare tail items that genuinely need them. This is precisely the asymmetry predicted by PIBA (Theorem~\ref{thm:optimal}) and confirms that the efficiency gains are a structural consequence of popularity-aware length allocation rather than a uniform truncation. Together with the latency results in Table~\ref{tab:efficiency}, this analysis shows that VarLenRec improves recommendation accuracy and inference efficiency simultaneously, rather than trading one for the other.

\section{Conclusion and Future Work}
\label{sec:conclusion}

In this work, we identified a fundamental limitation in generative recommendation: uniform-length tokenization fails to account for heterogeneous information requirements across items with different popularity levels. Through systematic empirical analysis, we discovered the Popularity-Length Paradox, where popular items achieve optimal performance with short IDs while tail items require longer codes. To address this, we proposed VarLenRec, a principled framework for learning variable-length semantic IDs. Our theoretical contribution, Popularity-Weighted Information Budget Allocation (PIBA), establishes that optimal ID length should scale as a negative power of popularity, and our geometric analysis shows that hyperbolic space provides exponential, distortion-free capacity stratification matched to this allocation. Building on this foundation, we introduced HARQ to stratify encoding capacity and a Length Predictor that maps item content to its target length. Extensive experiments demonstrate that VarLenRec achieves substantial improvements in recommendation accuracy, collision reduction, and computational efficiency, validating that adaptive encoding capacity is essential for effective generative recommendation.

Our work opens several promising directions for future research. Extending variable-length tokenization to multimodal and cross-domain settings could further enhance semantic expressiveness. Developing adaptive allocation strategies beyond popularity, such as category-specific or temporal patterns, may yield additional gains. Investigating the interplay between variable-length IDs and emerging large language model architectures presents an exciting avenue for next-generation recommender systems. We hope VarLenRec serves as a foundation for adaptive tokenization in generative recommendation and inspires future work on personalized encoding strategies.

\bibliographystyle{IEEEtran}
\bibliography{main}

\end{document}